%% file: arxiv.tex
\documentclass{article}

\usepackage[verbose=true,letterpaper]{geometry}

\usepackage{mleftright}\mleftright

\input{package}
\input{macro}
\usepackage{svg}
\usepackage{fullpage}
\usepackage{libertine}

\usepackage[color={red!100!green!33},colorinlistoftodos,prependcaption,textsize=small]{todonotes}

\newcommand{\shortlongproofs}[2]{{#2}}

\title{Sample Complexity of Learning Heuristic Functions \\ for Greedy-Best-First and A* Search}

\author{%
Shinsaku Sakaue\\
The University of Tokyo\\
Tokyo, Japan\\
\href{mailto:sakaue@mist.i.u-tokyo.ac.jp}{sakaue@mist.i.u-tokyo.ac.jp} 
\and
Taihei Oki\\
The University of Tokyo\\
Tokyo, Japan\\
\href{mailto:oki@mist.i.u-tokyo.ac.jp}{oki@mist.i.u-tokyo.ac.jp} 
}

\begin{document}

\maketitle

\begin{abstract}
Greedy best-first search (GBFS) and A* search (A*) are popular algorithms for path-finding on large graphs. Both use so-called heuristic functions, which estimate how close a vertex is to the goal. While heuristic functions have been handcrafted using domain knowledge, recent studies demonstrate that learning heuristic functions from data is effective in many applications. Motivated by this emerging approach, we study the sample complexity of learning heuristic functions for GBFS and A*. We build on a recent framework called \textit{data-driven algorithm design} and evaluate the \textit{pseudo-dimension} of a class of utility functions that measure the performance of parameterized algorithms. Assuming that a vertex set of size $n$ is fixed, we present $\mathrm{O}(n\lg n)$ and $\mathrm{O}(n^2\lg n)$ upper bounds on the pseudo-dimensions for GBFS and A*, respectively, parameterized by heuristic function values. The upper bound for A* can be improved to $\mathrm{O}(n^2\lg d)$ if every vertex has a degree of at most $d$ and to $\mathrm{O}(n \lg n)$ if edge weights are integers bounded by $\mathrm{poly}(n)$. We also give $\Omega(n)$ lower bounds for GBFS and A*, which imply that our bounds for GBFS and A* under the integer-weight condition are tight up to a $\lg n$ factor. Finally, we discuss a case where the performance of A* is measured by the suboptimality and show that we can sometimes obtain a better guarantee by combining a parameter-dependent worst-case bound with a sample complexity bound. 
\end{abstract}

\newcommand{\Pdim}{\mathrm{Pdim}}
\newcommand{\VCdim}{\mathrm{VCdim}}
\newcommand{\Ibb}{\mathbb{I}}

\newcommand{\bdf}{b}
\newcommand{\pf}{f}
\newcommand{\nbd}{K} 

\newcommand{\rhob}{{\bm{\rho}}}
\newcommand{\xinst}{x}
\newcommand{\wset}{\set*{w_e}_{e \in E}}

\newcommand{\opt}{\mathrm{Opt}}
\newcommand{\cost}{\mathrm{Cost}}

\newcommand{\astar}{A*\xspace}
\newcommand{\gbfs}{GBFS\xspace}

\newcommand{\open}{\texttt{OPEN}}
\newcommand{\closed}{\texttt{CLOSED}}
\newcommand{\pointer}{\texttt{p}}

\newcommand{\vc}{c}

\section{Introduction}\label{sec:introduction}
Given a graph with a start vertex $s$, a goal vertex $t$, and non-negative edge weights, we consider finding an $s$--$t$ path with a small total weight. 
The Dijkstra algorithm \citep{Dijkstra1959-ai} finds an optimal path by exploring all vertices that are as close to $s$ as $t$.
It, however, is sometimes impractical for large graphs since exploring all such vertices is too costly. 
Heuristic search algorithms are used to address such situations; 
among them, greedy best-first search (\gbfs) \citep{Doran1966-pj} and \astar search (\astar) \citep{Hart1968-jm} are two popular algorithms.  
Both \gbfs and \astar use so-called heuristic functions, which estimate how close an input vertex is to $t$. 
\gbfs/\astar attempts to avoid redundant exploration by scoring vertices based on heuristic function values and iteratively expanding vertices with the smallest score. 
If well-suited heuristic functions are available, \gbfs/\astar can run much faster than the Dijkstra algorithm. 
Furthermore, if \astar uses an \textit{admissible} heuristic function, i.e., it never overestimates the shortest-path distance to $t$, it always finds an optimal path \citep{Hart1968-jm}. 
Traditionally, heuristic functions have been made based on domain knowledge; e.g., if graphs are road networks, the Euclidean distance gives an admissible heuristic. 

When applying \gbfs/\astar to various real-world problems, a laborious process is to handcraft heuristic functions. 
Learning heuristic functions from data can be a promising approach to overcoming the obstacle due to the recent development of technologies for collecting graph data. 
Reseachers have demonstrated the effectiveness of this approach in robotics \citep{Bhardwaj2017-lk,Takahashi2019-il,Pandy2021-wn,Yonetani2021-wn}, computational organic chemistry \citep{Chen2020-xa}, and predestrian trajectory prediction \citep{Yonetani2021-wn}. 
With learned heuristic functions, however, obtaining theoretical guarantees is difficult since we can hardly understand how the search can be guided by such heuristic functions. 
(A recent paper \citep{Agostinelli2021-dy} studies learning of admissible heuristics for \astar, but the optimality is confirmed only empirically.) 
Moreover, learned heuristic functions may be overfitting to problem instances at hand. 
That is, even if \gbfs/\astar with learned heuristic functions perform well over training instances, they may deliver poor future performance. 
In summary, the emerging line of work on search algorithms with learned heuristic functions is awaiting a theoretical foundation for guaranteeing their performance in a data-driven manner. 
Thus, a natural question is: 
\textit{how many sampled instances are needed to learn heuristic functions with generalization guarantees on the performance of resulting \gbfs/\astar?}  

\subsection{Our contribution}\label{subsec:contribution}
We address the above question, assuming that path-finding instances defined on a fixed vertex set of size $n$ are drawn i.i.d. from an unknown distribution. 
Our analysis is based on so-called \textit{data-driven algorithm design} \citep{Gupta2017-ng,Balcan2021-fy}, a PAC-learning framework for bounding the sample complexity of algorithm configuration. 
In the analysis, the most crucial step is to evaluate the \textit{pseudo-dimension} of a class of utility functions that measure the performance of parameterized algorithms. 
We study the case where \gbfs/\astar is parameterized by heuristic function values and make the following contributions: 

\begin{enumerate}
	\item 
	\Cref{sec:upper_bound} gives $\Ord(n\lg n)$ and $\Ord(n^2\lg n)$ upper bounds on the pseudo-dimensions for \gbfs and \astar, respectively. 
	The bound for \astar can be improved to $\Ord(n^2\lg d)$ if every vertex has an at most $d$ degree and to $\Ord(n \lg n)$ if edge weights are non-negative integers at most $\poly(n)$. 	
	\item \Cref{sec:lower_bound} presents $\Omega(n)$ lower bounds on the pseudo-dimensions for \gbfs and \astar. 
	We prove this result by constructing $\Omega(n)$ instances with unweighted graphs. 
	Thus, our bounds for GBFS and A* under the integer edge-weight condition are tight up to a $\lg n$ factor.
	\item \Cref{sec:suboptimality_astar} studies a particular case of bounding the suboptimality of \astar. 
	We show that we can sometimes improve the guarantee obtained in \Cref{sec:upper_bound} by using an alternative $\Ord(n \lg n)$ bound on the pseudo-dimension of a class of parameter-dependent worst-case bounds \citep{Valenzano2014-sw}.  
\end{enumerate}

An important consequence of the above results is the tightness up to a $\lg n$ factor for \gbfs and \astar under the integer-weight assumption. 
Note that this assumption holds in various realistic situations. 
For example, the Internet network and state-space graphs of games are unweighted (unit-weight) graphs, and \astar is often applied to path-finding instances on such graphs.

\subsection{Related work}\label{subsec:related_work}
\paragraph{Data-driven algorithm design.}
\citet{Gupta2017-ng} proposed a PAC approach for bounding the sample complexity of algorithm configuration, which is called \textit{data-driven algorithm design} and has been applied to a broad family of algorithms, including greedy, clustering, and sequence alignment algorithms. 
We refer the reader to a nice survey \citep{Balcan2021-fy}. 
A recent line of work \citep{Balcan2018-pe,Balcan2021-kv,Balcan2021-fz} has extensively studied the sample complexity of configuring integer-programming methods, e.g., branch-and-bound and branch-and-cut. 
In \citep{Balcan2021-kv, Balcan2021-fz}, upper bounds on the pseudo-dimension for general tree search are presented, which are most closely related to our results. 
Our upper bounds, which are obtained by using specific properties of \gbfs/\astar, are better than the previous bounds for general tree search, as detailed in \Cref{app_sec:comparison}. 
\citet{Balcan2021-jv} presented a general framework for evaluating the pseudo-dimension. 
Their idea is to suppose that performance measures form a class of functions of algorithm parameters, called \textit{dual} functions, and characterize its complexity based on how they are piecewise structured. 
This idea plays a key role in the analysis of \citep{Balcan2021-kv,Balcan2021-fz}, and our analysis of the upper bounds are also inspired by their idea. 
Its application to our setting, however, requires a close look at the behavior of \gbfs/\astar. 
\citet{Balcan2020-gm} showed that approximating dual functions with simpler ones is useful for improving sample complexity bounds, which is similar to our idea in \Cref{sec:suboptimality_astar}. 
A difference is that while they construct simpler functions with a dynamic programming algorithm, we can use a known worst-case bound on the suboptimality of best-first search \citep{Valenzano2014-sw}. 
Lower bounds on the pseudo-dimension for graph-search algorithms have not been well studied. 

\paragraph{Heuristic search with learning.}
\citet{Eden2022-fl} theoretically studied how the average-case running time of \astar can be affected by the dimensions or bits of learned embeddings or labels of vertex features, based on which heuristic function values and computed. 
The sample complexity of learning heuristic functions, however, has not been studied. 


\section{Preliminaries}\label{sec:problem_formulation}
We present the background on learning theory and our problem setting. 
In what follows, we let $\Ibb\prn*{\cdot}$ be a boolean function that returns $1$ if its argument is true and $0$ otherwise. 
We use $\Hcal\subseteq \Rcal^\Ycal$ to denote a class of functions that map $\Ycal$ to $\Rcal \subseteq \R$.
For any positive integer $m$, we let $[m] = \set*{1,\dots,m}$. 

\subsection{Backgound on learning theory}
The following \textit{pseudo-dimension} \citep{Pollard1984-zp} is a fundamental notion for quantifying the complexity of a class of real-valued functions. 

\begin{definition}
	Let $\Hcal \subseteq \R^\Ycal$ be a class of functions that map some domain $\Ycal$ to $\R$. 
	We say a set $\set*{y_1,\dots,y_N}\subseteq \Ycal$ is \textit{shattered} by $\Hcal$ if there exist target values, $t_1,\dots,t_N \in \R$, such that 
	\[
		\abs*{ \Set*{ \prn*{ \Ibb\prn*{h(y_1) \ge z_1}, \dots, \Ibb\prn*{h(y_N) \ge z_N} } }{ h \in \Hcal } } = 2^N.
	\]
	The \emph{pseudo-dimension} of $\Hcal$, denoted by $\Pdim(\Hcal)$, is the size of a largest set shattered by $\Hcal$.	
\end{definition}
If $\Hcal$ is a set of binary-valued functions that map $\Ycal$ to $\set*{0, 1}$,
the pseudo-dimension of $\Hcal$ coincides with the so-called \textit{VC-dimension}~\citep{Vapnik1971-me},
which is denoted by $\VCdim(\Hcal)$.

The following proposition enables us to obtain sample complexity bounds by evaluating the pseudo-dimension (see, e.g., \citep[Theorem 13.6]{Anthony1999-mm} and \citep[Theorem 11.8]{Mohri2018-zs}). 

\begin{proposition}\label{prop:complexity_bound}
	Let $H>0$, $\Hcal \subseteq {[0, H]}^\Ycal$, and $\Dcal$ be a distribution over $\Ycal$. 
	For any $\delta\in(0,1)$, with a probability of at least $1-\delta$ over the i.i.d.\ draw of $\set*{y_1,\dots,y_N} \sim \Dcal^N$, for all $h \in \Hcal$, it holds that
	\[
    \abs*{ \frac{1}{N}\sum_{i=1}^N h(y_i)  - \mathop{\E}_{y \sim \Dcal}\brc*{ h(y) } } = \Ord\prn*{H \sqrt{\frac{\Pdim(\Hcal) \lg \frac{N}{\Pdim(\Hcal)} + \lg \frac{1}{\delta}}{N}}}. 
	\] 
\end{proposition}
In other words, for any $\epsilon>0$, $N = \Omega\prn*{\frac{H^2}{\epsilon^2} \prn*{\Pdim(\Hcal) \lg\frac{H}{\epsilon} + \lg \frac{1}{\delta}}}$ sampled instances are sufficient to ensure that with a probability of at least $1-\delta$, for all $h \in \Hcal$, the difference between the empirical average and the expectation over an unknown disrtibution $\Dcal$ is at most $\epsilon$.

\subsection{Problem formulation}\label{subsec:problem}
We describe path-finding instances, \gbfs/\astar algorithm, and performance measures considered in this paper. 

\paragraph{Path-finding instances.}
We consider solving randomly generated path-finding instances repetitively. 
Let $\xinst = (V, E, \set*{w_e}_{e \in E}, s, t)$ be a path-finding instance, where $(V, E)$ is a simple directed graph with $n$ vertices, $\set*{w_e}_{e \in E}$ is a set of non-negative edge weights (sometimes called costs), $s\in V$ is a start vertex, and $t\in V$ is a goal vertex. 
We let $\Pi$ be a class of possible instances. 
Each instance $\xinst\in\Pi$ is drawn from an unknown distribution $\Dcal$ over $\Pi$. We impose the following assumption on $\Pi$. 

\begin{assumption}\label{assump:feasible}
	For all $\xinst \in \Pi$, the vertex set $V$ and the goal node $t$ are identical, and there always exists at least one directed path from $s \neq t$ to $t$, i.e., every instance $\xinst \in \Pi$ is feasible. 
\end{assumption}

Fixing $V$ is necessary for evaluating the pseudo-dimension in terms of $n = |V|$. 
Note that we can deal with the case where some instances in $\Pi$ are defined on vertex subsets $V' \subseteq V$ by removing edges adjacent to $V \setminus V'$. 
The feasibility assumption is needed to ensure that \gbfs/\astar always returns a solution, and $s\neq t$ simply rules out the trivial case where the empty set is optimal.  
In \Cref{app_sec:additional}, we discuss how to extend our results to the case where $t$ can change depending on instances.

\paragraph{Algorithm description.}
We sketch algorithmic procedures that are common to both \gbfs and \astar (see \Cref{alg:gbfs,alg:astar} for details, respectively).  
Let $A_\rhob$ be a \gbfs/\astar algorithm, which is parameterized by heuristic function values $\rhob \in \R^n$. 
Given an instance $\xinst\in\Pi$, $A_\rhob$ starts from $s$ and iteratively builds a set of candidate paths. 
These paths are maintained by $\open$ and $\closed$ lists, together with pointers $\pointer(\cdot)$ to parent vertices. 
The $\open$ list contains vertices to be explored, and the $\closed$ list consists of vertices that have been explored. 
In each iteration, we select a vertex $v$ from $\open$, expand $v$, and move $v$ from $\open$ to $\closed$. 

Heuristic function values $\rhob$ are used when selecting vertices. 
For each $v \in V$, the corresponding entry in $\rhob$, denoted by $\rho_v$, represents an estimated shortest-path distance from $v$ to $t$. (Although heuristic function values are usually denoted by $h(v)$, we here use $\rho_v$ for convenience.) 
In each iteration, we select a vertex with the smallest \textit{score}, which is defined based on $\rhob$ as detailed later. 
We impose the following assumption on the vertex selection step. 

\begin{assumption}\label{assump:tie_break}
	Define an arbitrary strict toral order on $V$; for example, we label elements in $V$ by $v_1,\dots, v_n$ and define a total order $v_1<\dots<v_n$. 
	When selecting a vertex with the smallest score, we break ties, if any, in favor of the smallest vertex with respect to the total order. 
\end{assumption}

If we allow $A_\rhob$ to break ties arbitrarily, its behavior becomes too complex to obtain meaningful bounds on the pseudo-dimension. 
\Cref{assump:tie_break} is a natural rule to exclude such troublesome cases.

\paragraph{Performance measure.}
Let $A_\rhob$ be \gbfs/\astar with parameters $\rhob \in \R^n$. 
We measure performance of $A_\rhob$ on $\xinst \in \Pi$ with a utility function $u$. 
We assume $u$ to satisfy the following condition. 

\begin{assumption}\label{assump:utility}
	Let $H > 0$. 
	A utility function $u$ takes $x$ and a series of all $\emph{\open}$, $\emph{\closed}$, and $\emph{\pointer}(\cdot)$ generated during the execution of $A_\rhob$ on $\xinst \in \Pi$ as input, and returns a scalar value in $[0, H]$. 	
\end{assumption}

We sometimes use $A_\rhob$ to represent the series of $\open$ and $\closed$ lists and pointers generated by $A_\rhob$. 
Note that $u$ meeting \Cref{assump:utility} can measure various kinds of performance. 
For example, since the pointers indicate an $s$--$t$ path returned by $A_\rhob$, $u$ can represent its cost. 
Moreover, since the series of $\open$ and $\closed$ lists maintain all search states, $u$ can represent the time and space complexity of $A_\rhob$. 
We let $u_\rhob: \Pi \to [0, H]$ denote the utility function that returns the performance of $A_\rhob$ on any $\xinst\in\Pi$, and define a class of such functions as $\Ucal = \Set*{u_\rhob: \Pi \to [0, H]}{\rhob\in \R^n}$. 
The upper bound, $H$, is necessary to obtain sample complexity bounds with \Cref{prop:complexity_bound}. 
Setting such an upper bound is usual in practice. 
For example, if $u$ measures the running time, $H$ represents a time-out deadline. 

\paragraph{Generalization guarantees on performance.}
Given the above setting, we want to learn $\hat\rhob$ values that attain an optimal $\E_{\xinst\sim\Dcal}[u_{\hat\rhob}(\xinst)]$ value, where available information consists of sampled instances $\xinst_1,\dots,\xinst_N$ and $u_\rhob(\xinst_1), \dots, u_\rhob(\xinst_N)$ values for any $\rhob\in\R^n$. 
To obtain generalization gaurantees on the performance of $A_{\hat\rhob}$, we bound $\abs{\frac{1}{N}\sum_{i=1}^N u_\rhob(\xinst_i) - \E_{\xinst\sim \Dcal}[u_\rhob(x)]}$ uniformly for all $\rhob\in \R^n$. 
Note that the uniform bound offers performance guarantees that are independent of learning procedures, e.g., manual or automated (without being uniform, learned $\hat\rhob$ may be overfitting sampled instances). 
As in \Cref{prop:complexity_bound}, to bound the sample complexity of learning $\rhob$ values, we need to evaluate the pseudo-dimension of $\Ucal$, denoted by $\Pdim(\Ucal)$, which is the main subject of this study.

\paragraph{Remarks on heuristic functions.} 
While we allow heuristic function values $\rhob$ to be any point in $\R^n$, the range of heuristic functions may be restricted to some subspace of $\R^n$. 
Note that our upper bounds are applicable to such situations since restricting the space of possible $\rhob$ values does not increase $\Pdim(\Ucal)$. 
Meanwhile, such restriction may be useful for improving the upper bounds on $\Pdim(\Ucal)$; exploring this direction is left for future work. 
Also, our setting cannot deal with heuristic functions that take some instance-dependent features as input. 
To study such cases, we need more analysis that is specific to heuristic function models, which goes beyond the scope of this paper. Thus, we leave this for future work. 
Note that our setting still includes important heuristic function models on fixed vertex sets. 
For example, we can set $\rhob$ using learned distances to landmarks \citep{Goldberg2005-lj}, 
or we can let $\rhob$ be distances measured on some metric space by learning metric embeddings of vertices \citep{You2019-hs}.

\section{Upper bounds on the pseudo-dimension}\label{sec:upper_bound}

We present details of \gbfs and \astar and upper bounds on the pseudo-dimensions of $\Ucal$. 
In this section, we suppose that vertices in $V$ are labeled by $v_1,\dots,v_n$ as in \Cref{assump:tie_break}. 

\begin{algorithm}[tb]
	\caption{\gbfs with heuristic function values $\rhob$}
	\label{alg:gbfs}
	\begin{algorithmic}[1]
		\State $\open = \set*{s}$, $\closed = \emptyset$, and $\pointer(s) = \mathrm{None}$.
		\While{$\open$ is not empty}
    \State $v \gets \argmin\Set*{\rho_{v'}}{v' \in \open}$. {\label[step]{step:gbfs_argmin}}
      \Comment{Break ties as in \Cref{assump:tie_break}.}
    \For {each child $\vc$ of $v$}
		\If {$\vc = t$}
		\State \Return $s$--$t$ path by tracing pointers $\pointer(\cdot)$, where $\pointer(t) = v$.
		\EndIf
    \If {$\vc \notin \open \cup \closed$}
    \State $\pointer(\vc) \gets v$ and $\open \gets \open \cup \set*{\vc}$.
    \EndIf
    \EndFor
		\State Move $v$ from $\open$ to $\closed$.
		\EndWhile
	\end{algorithmic}
\end{algorithm}

\subsection{Greedy best-first search}

\Cref{alg:gbfs} shows the details of \gbfs $A_\rhob$ with heuristic function values $\rhob \in \R^n$. 
When selecting vertices in \Cref{step:gbfs_argmin}, the scores are determined only by $\rhob$. 
This implies an obvious but important fact. 

\begin{lemma}\label{lem:gbfs_total_order}
	Let $\rhob, \rhob' \in \R^n$ be a pair of heuristic function values with an identical total order up to ties on their entries, i.e., $\Ibb(\rho_{v_i} \le \rho_{v_j}) = \Ibb(\rho'_{v_i} \le \rho'_{v_j})$ for all $i,j\in [n]$ such that $i < j$. 
	Then, we have $u_\rhob(\xinst) = u_{\rhob'}(\xinst)$ for all $\xinst\in\Pi$. 
\end{lemma}

\begin{proof}
	For any $\xinst\in\Pi$, if $\rhob$ and $\rhob'$ have an identical strict total order on their entries, vertices selected in \Cref{step:gbfs_argmin} are the same in each iteration of $A_\rhob$ and $A_{\rhob'}$. 
	Since this is the only step $\rhob$ and $\rhob'$ can affect, we have $A_\rhob = A_{\rhob'}$ for all $\xinst\in\Pi$, hence $u_\rhob(\xinst) = u_{\rhob'}(\xinst)$. Moreover, this holds even if $\rhob$ and/or $\rhob'$ have ties on their entries because of \Cref{assump:tie_break}. 
	That is, the total order uniquely determines a vertex selected in \Cref{step:gbfs_argmin} even in case of ties. 
	Therefore, the statement holds. 
\end{proof}

From \Cref{lem:gbfs_total_order}, the behavior of \gbfs is uniquely determined once a total order on $\set*{\rho_v}_{v\in V}$ is fixed. 
Thus, for any $\xinst \in \Pi$, the number of distinct  $u_\rhob(\xinst)$ values is at most $n!$, the number of permutations of $\set*{\rho_v}_{v\in V}$. 
This fact enables us to obtain an $\Ord(n \lg n)$ upper bound on the pseudo-dimension of $\Ucal$. 

\begin{theorem}\label{thm:gbfs_upper}
	For \gbfs $A_\rhob$ with parameters $\rhob \in \R^n$, it holds that $\Pdim(\Ucal) = \Ord(n \lg n)$. 
\end{theorem}

\begin{proof}
	\Cref{lem:gbfs_total_order} implies that we can parition $\R^n$ into $n!$ regions, $\Pcal_1, \Pcal_2,\dots$, so that for every $\Pcal_i$, any pair of $\rhob, \rhob' \in \Pcal_i$ satisfies $u_\rhob(\xinst) = u_{\rhob'}(\xinst)$ for all $\xinst \in \Pi$. 
	Note that the construction of the regions, $\Pcal_1, \Pcal_2,\dots$, does not depend on $\xinst$. 
	Thus, given any $N$ instances $\xinst_1,\dots,\xinst_N$, even if $\rhob$ moves over whole $\R^n$, 
	the number of distinct tuples of form $(u_\rhob(\xinst_1),\dots,u_\rhob(\xinst_N))$ is at most $n!$. 
	To shatter $N$ instances, $n! \ge 2^N$ must hold. Solving this for the largest $N$ yields $\Pdim(\Ucal) = \Ord(n \lg n)$. 
\end{proof}

\subsection{\astar search}\label{subsec:astar_upper}

\begin{algorithm}[tb]
	\caption{\astar with heuristic function values $\rhob$}
	\label{alg:astar}
	\begin{algorithmic}[1]
		\State $\open = \set*{s}$, $\closed = \emptyset$, $\pointer(s) = \mathrm{None}$, and $g_s = 0$.
		\While{$\open$ is not empty}
    \State $v \gets \argmin\Set*{g_{v'} + \rho_{v'}}{v' \in \open}$. {\label[step]{step:astar_argmin}}
      \Comment{Break ties as in \Cref{assump:tie_break}.}
    \If {$v = t$} 
		\State \Return $s$--$t$ path by tracing pointers $\pointer(\cdot)$.
    \EndIf
    \For {each child $\vc$ of $v$}
    \State $g_\mathrm{new} \gets g_v + w_{(v, \vc)}$.
    \If {$\vc \notin \open \cup \closed$}
    \State $g_{\vc} \gets g_\mathrm{new}$, $\pointer(\vc) \gets v$, and $\open \gets \open \cup \set*{\vc}$.
    \ElsIf {$\vc \in \open$ and $ g_\mathrm{new} < g_{\vc}$}
    \State $g_{\vc} \gets g_\mathrm{new}$ and $\pointer(\vc) \gets v$.
    \ElsIf {$\vc \in \closed$ and $g_\mathrm{new} < g_{\vc}$} {\label[step]{step:reopening_if}}
			\Comment{\Cref{step:reopening_if,step:parent_pointer_updating,step:reopening_o} are for reopening.}
    \State $g_{\vc} \gets g_\mathrm{new}$ and $\pointer(\vc) \gets v$. {\label[step]{step:parent_pointer_updating}}
    \State Move $\vc$ from $\closed$ to $\open$. {\label[step]{step:reopening_o}}
    \EndIf
    \EndFor
		\State Move $v$ from $\open$ to $\closed$.
		\EndWhile
	\end{algorithmic}
\end{algorithm}

\Cref{alg:astar} is the details of \astar. 
As with \gbfs, $\rhob$ only affects the vertex selection step (\Cref{step:astar_argmin}). 
However, unlike \gbfs, the scores, $g_v + \rho_v$, involve not only $\rhob$ but also $\set*{g_v}_{v\in V}$. 
Each $g_v$ is called a $g$-cost and maintains a cost of some path from $s$ to $v$. 
As in \Cref{alg:astar}, when $v$ is expanded and a shorter path to $\vc$ is found, whose cost is denoted by $g_{\textrm{new}}$, we update the $g_{\vc}$ value. 
Thus, each $g_{v}$ always gives an upper bound on the shortest-path distance from $s$ to $v$. 
For each $v \in V$, there are at most $\sum_{k=0}^{n-2}k! \le (n-1)!$ simple paths connecting $s$ to $v$, and thus $g_v$ can take at most $(n-1)!$ distinct values. 
We denote the set of those distinct values by $\Gcal_v$, 
and define $\Gcal_V = \Set*{(v, g_v)}{v\in V, g_v \in \Gcal_v}$ as the set of all pairs of a vertex and its possible $g$-cost. It holds that $|\Gcal_V| \le n\times (n-1)! = n!$. 

Note that once $\xinst \in \Pi$ is fixed, $\Gcal_v$ for $v\in V$ and $\Gcal_V$ are uniquely determined. 
To emphasize this fact, we sometimes use notation with references to $\xinst$: 
$g_v(\xinst)$, $\Gcal_v(\xinst)$, and $\Gcal_V(\xinst)$.  
As with the case of \gbfs (\Cref{lem:gbfs_total_order}), we can define a total order on the scores to determine the behavior of \astar uniquely. 

\begin{lemma}\label{lem:astar_total_order}
	Fix any instance $\xinst \in \Pi$. 
	Let $\rhob, \rhob' \in \R^n$ be a pair of heuristic function values such that total orders on 
	the sets of all possible scores, 
	$\Set*{g_v(\xinst) + \rho_v}{(v, g_v(\xinst)) \in \Gcal_V(\xinst)}$ and 
	$\Set*{g_v(\xinst) + \rho'_v}{(v, g_v(\xinst)) \in \Gcal_V(\xinst)}$, are identical up to ties.  
	Then, it holds that $u_\rhob(\xinst) = u_{\rhob'}(\xinst)$. 
\end{lemma}

\begin{proof}
	If the two sets of scores have an identical strict total order, we select the same vertex in \Cref{step:astar_argmin} in each iteration of $A_\rhob$ and $A_{\rhob'}$. Thus, we have $A_\rhob = A_{\rhob'}$ for any fixed $\xinst$, implying $u_\rhob(\xinst) = u_{\rhob'}(\xinst)$. 
	We show that this holds even in the presence of ties by using \Cref{assump:tie_break}. 
	First, any two scores of the same vertices, $g_v(\xinst) +\rho_v$ and $g_v'(\xinst) +\rho_v$, never have ties since $\Gcal_v$ consists of distinct $g$-costs. 
	Next, if $g_{v_i}(\xinst) +\rho_{v_i} = g_{v_j}(\xinst) +\rho_{v_j}$ holds for some $i < j$, we always prefer $v_i$ to $v_j$ in \Cref{step:astar_argmin} due to \Cref{assump:tie_break}. 
	Therefore, even in the presence of ties, we select a vertex in \Cref{step:astar_argmin} as if the set of scores has a strict total order. 
	Thus, if $\rhob$ and $\rhob'$ induce the same total order up to ties on the sets of possible scores, it holds that $u_\rhob(\xinst) = u_{\rhob'}(\xinst)$. 
\end{proof}

By using \Cref{lem:astar_total_order}, we can obtain an $\Ord(n^2\lg n)$ upper bound on the pseudo-dimension of $\Ucal$. 

\begin{theorem}\label{thm:astar_upper}
	For \astar $A_\rhob$ with parameters $\rhob \in \R^n$, it holds that $\Pdim(\Ucal) = \Ord(n^2 \lg n)$. 
\end{theorem}

\begin{proof}
	As with the proof of \Cref{thm:gbfs_upper}, we partition $\R^n$ into some regions so that in each region, the behavior of \astar is unique. 
	Unlike the case of \gbfs, boundaries of such regions change over $N$ instances. 
	To deal with this situation, we use a geometric fact: 
	for $m \ge n \ge 1$, $m$ hyperplanes partition $\R^n$ into $\Ord({(\e m)}^n)$ regions.\footnote{Even if some regions degenerate, from \citep[Theorem 28.1.1]{Halperin2017-lp} and \citep[Proposition A2.1]{Blumer1989-cj}, the number of all $d$-dimensional regions for $d=0,\dots,n$ is $\sum_{d=0}^n \sum_{i=0}^d\binom{n - i}{d - i}\binom{m}{n - i} \le 2{(\e m)}^n$. 
	The fact has a close connection to Sauer's lemma \citep{Sauer1972-pw} (see \citep{Gartner1994-qi}). In this sense, our analysis is in a similar spirit to the general framework of \citep{Balcan2021-jv}. 
	}

	Fix a tuple of any $N$ instances $(\xinst_1,\dots,\xinst_N)$. 
	We consider hyperplanes in $\R^n$ of form $g_{v_i}(\xinst_k) + \rho_{v_i} = g_{v_j}(\xinst_k) + \rho_{v_j}$ for all $k\in[N]$ and all pairs of $(v_i,g_{v_i}(\xinst_k)), (v_j,g_{v_j}(\xinst_k)) \in \Gcal_V$ such that $i\neq j$. 
	These hyperplanes partition $\R^n$ into some regions, $\Pcal_1,\Pcal_2,\dots$, so that the following condition holds: 
	for every $\Pcal_i$, any $\rhob, \rhob' \in \Pcal_i$ have the same total order on $\Set*{g_v(\xinst_k) + \rho_v}{(v, g_v(\xinst_k)) \in \Gcal_V(\xinst)}$ and $\Set*{g_v(\xinst_k) + \rho'_v}{(v, g_v(\xinst_k)) \in \Gcal_V(\xinst_k)}$ up to ties for all $k \in [N]$, which implies $u_\rhob(\xinst_k) = u_{\rhob'}(\xinst_k)$ for all $k \in [N]$ due to \Cref{lem:astar_total_order}. 
	That is, for every $k \in [N]$, if we see $u_\rhob(\xinst_k)$ as a function of $\rhob$, it is piecewise constant where pieces are given by $\Pcal_1,\Pcal_2,\dots$. 
	Therefore, when $\rhob$ moves over whole $\R^n$, 
	the number of distinct tuples of form $\prn*{u_\rhob(\xinst_1), \dots, u_\rhob({\xinst_N})}$ is at most the number of the pieces. 
	Note that the pieces are generated by partitioning $\R^n$ with $\sum_{k\in[N]} \binom{|\Gcal_V(\xinst_k)|}{2} \le N \binom{n!}{2}$ hyperplanes, which means there are at most $\Ord\prn*{\prn*{\e N \binom{n!}{2}}^n}$ pieces. 
	To shatter $N$ instances, $\Ord\prn*{\prn*{\e N \binom{n!}{2}}^n} \ge 2^N$ is necessary. 
	Solving this for the largest $N$ yields $\Pdim(\Ucal) = \Ord(n^2 \lg n)$.  
\end{proof}

Compared with \gbfs, the additional $n$ factor comes from the bound of $(n-1)!$ on $|\Gcal_v|$. 
This bound may seem too pessimistic, but it is almost tight in some cases, as implied by the following example. 

\begin{example}\label{exmp:large_Gv}
	Let $(V, E)$ be a complete graph with edges labeled as $\set*{e_1,\dots,e_{|E|}}$. 
	Set each edge weight $w_{e_i}$ to $2^{i-1}$ for $i\in[|E|]$. 
	Considering the binary representation of the edge weights, the costs of all simple $s$--$v$ paths are mutually different for $v\in V$, which implies $|\Gcal_v| = \sum_{k=0}^{n-2}k! \ge (n-2)!$. 	
\end{example}

This example suggests that improving the $\Ord(n^2\lg n)$ bound is not straightforward. 
Under some realistic assumptions, however, we can improve it by deriving smaller upper bounds on $|\Gcal_v|$. 

First, if the maximum degree of vertices is always bounded, we can obtain the following bound.

\begin{theorem}
	Assume that the maximum out-degrees of directed graphs $(V, E)$ of all instances in $\Pi$ are upper bounded by $d$. Then, it holds that $\Pdim(\Ucal) = \Ord(n^2\lg d)$. 
\end{theorem}

\begin{proof}
	Under the assumption on the maximum degree, there are at most $\sum_{k=0}^{n-2}d^k \le (n-1)d^{n-2}$ simple $s$--$v$ paths, which implies $|\Gcal_v|\le (n-1)d^{n-2}$ for every $v \in V$. 
	Therefore, we have $|\Gcal_V| \le n\times (n-1)d^{n-2}$. 
	Following the proof of \Cref{thm:astar_upper}, we can obtain an upper bound on $\Pdim(\Ucal)$ by solving $\Ord\prn*{N^n\binom{n(n-1)d^{n-2}}{2}^n} \ge 2^N$ for the largest $N$, 
	which yields $\Pdim(\Ucal) = \Ord(n^2\lg d)$.
\end{proof}

Second, if edge weights are non-negative integers bounded by $\ell$, we can obtain the following bound. 

\begin{theorem}\label{thm:astar_integer}
	Assume that edge weights $\set*{w_e}_{e\in E}$ of all instances in $\Pi$ are non-negative integers bounded by a constant $\ell$ from above. Then, it holds that $\Pdim(\Ucal) = \Ord(n\lg(n\ell))$.	
\end{theorem}

\begin{proof}
	Under the assumption, every $g$-cost $g_v$ takes a non-negative integer value at most $n\ell$. 
	Since $\Gcal_v$ consists of distinct $g$-cost values, $|\Gcal_v| \le n\ell$ holds, hence $|\Gcal_V| \le n^2\ell$. 
	Solving $\Ord\prn*{N^n\binom{n^2\ell}{2}^n} \ge 2^N$ for the largest $N$, we obtain $\Pdim(\Ucal) = \Ord(n\lg(n\ell))$.	
\end{proof}

Note that if $\ell = \Ord(\poly(n))$ holds, we have $\Pdim(\Ucal) = \Ord(n \lg n)$. 

\paragraph{On reopening.} 
\astar is usually allowed to reopen closed vertices as in \Cref{step:reopening_if,step:parent_pointer_updating,step:reopening_o}. 
This, however, causes $\Omega(2^{n})$ iterations in general \citep{Martelli1977-ij}, albeit always finite \citep{Valenzano2016-re}. 
A popular workaround is to simply remove \Cref{step:reopening_if,step:parent_pointer_updating,step:reopening_o}, and such \astar without reopening has also been extensively studied \citep{Valenzano2014-sw,Sepetnitsky2016-zo,Chen2019-nf,Chen2021-zq}. 
Note that our results are applicable to \astar both with and without reopening.

\section{Lower bounds on the pseudo-dimension}\label{sec:lower_bound}

We present lower bounds on the pseudo-dimension for \gbfs/\astar. 
We prove the result by constructing $\Omega(n)$ shatterable instances with unweighted graphs. 
Therefore, the $\Ord(n \lg n)$ upper bounds for \gbfs (\Cref{thm:gbfs_upper}) and \astar under the edge-weight assumption (\Cref{thm:astar_integer}) are tight up to a $\lg n$ factor. 

\begin{figure}[tb]
	\centering
	\includesvg[width=1.0\textwidth]{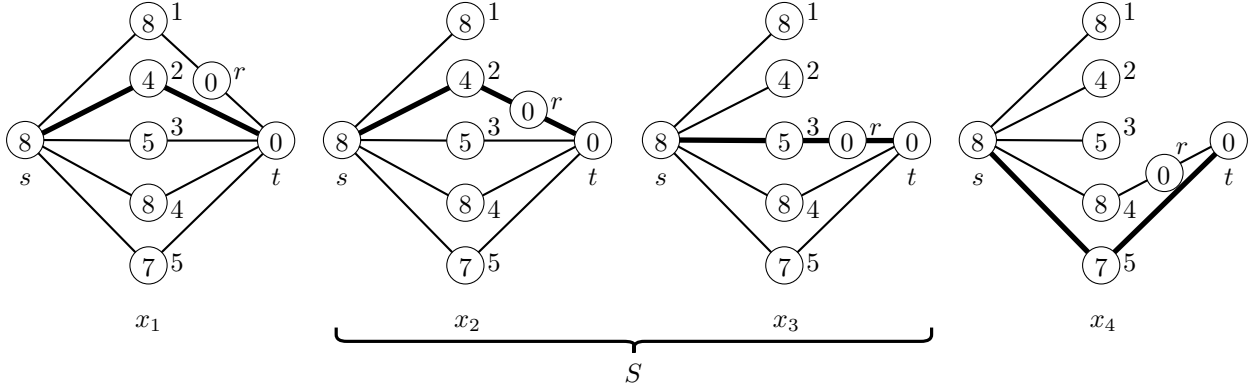}
	\caption{An illustration of the instances $\xinst_1,\dots,\xinst_{n-4}$ for $n=8$. 
	Each vertex is labeled by $s$, $r$, $t$, or $i \in [n-3]$, as shown nearby the vertex circles. 
	The values in vertex circles represent $\rhob$ that makes $A_\rhob$ return suboptimal paths to $\xinst_2$ and $\xinst_3$, i.e., $S = \set*{2, 3}$. 
	The thick edges indicate returned paths.}\label{fig:gbfs-lower}
\end{figure}

\begin{theorem}\label{thm:lower}
	For \gbfs/\astar $A_\rhob$ with parameters $\rhob \in \R^n$, it holds that $\Pdim(\Ucal) = \Omega(n)$. 
\end{theorem}
 
\shortlongproofs{
	\begin{proof}[Proof sketch]
		We construct a series of $n-4$ instances, $\xinst_{1},\dots,\xinst_{n-4}$, that can be shattered by $\Ucal$, where each $u_\rhob$ returns the length of an $s$--$t$ path found by $A_\rhob$. 
		We label vertices in $V$ by $s$, $r$, $t$, or $i \in [n-3]$. 
		See \Cref{fig:gbfs-lower} for an example with $n=8$. 
		We define $M = V\setminus \set*{s,r,t}$. 
		For each $\xinst_i$ ($i \in  [n-4]$), we draw edges $(s, v)$ for $v \in M$ and $(v, t)$ for $v \in \Set*{v' \in M}{v' > i}$, which constitute optimal $s$--$t$ paths of length $2$. 
		In addition, for each $\xinst_i$, we draw edges $(i, r)$ and $(r, t)$, where $s \to i \to r \to t$ is the only suboptimal path of length $3$. 
		Letting $t_i = 2.5$ for $i \in [n-4]$, we prove that $\Ucal$ can shatter those $n-4$ instances, i.e., $A_\rhob$ can return suboptimal solutions to any subset of $\set*{\xinst_1\dots,\xinst_{n-4}}$ by appropriately setting $\rhob$. 
	
		Let $S \subseteq [n-4]$ indicate a subset of instances, to which we will make $A_\rhob$ return suboptimal solutions. 
		We show that for any $S$, we can set $\rhob$ so that $A_\rhob$ returns $s \to i \to r \to t$ to $\xinst_i$ if and only if $i \in S$. 
		We refer to the vertex labeled by $n-3$ as $m$, which we use to ensure that every instance has an optimal path $s \to m \to t$. 
		We set $\rhob$ as follows: 
		$\rho_s = n$ (or an arbitrary value), 
		$\rho_r = \rho_t = 0$, 
		$\rho_i = i+2$ if $i \in S \cup \set*{m}$, 
		and $\rho_i = n$ (or a sufficiently large value) if $i \in [n-4]\setminus S$. 
		If $A_\rhob$ with this $\rhob$ is applied to $\xinst_i$, it iteratively selects vertices in $S \cup \set*{m}$ in increasing order of their labels until a vertex that has a child is selected. 
		Once a vertex with a child is expanded, it ends up returning $s \to i \to r \to t$ if $i \in S$ and $s \to v \to t$ for some $v > i$ if $i \notin S$.
		We detail this in the full proof presented in the supplementary. (As we will there, both \gbfs and \astar return the same $s$--$t$ paths.)
	\end{proof}
}{
	\begin{proof}
		We construct a series of $n-4$ instances, $\xinst_{1},\dots,\xinst_{n-4}$, that can be shattered by $\Ucal$, where each $u_\rhob$ returns the length of an $s$--$t$ path found by $A_\rhob$. 
		We label vertices in $V$ by $s$, $r$, $t$, or $i \in [n-3]$. 
		See \Cref{fig:gbfs-lower} for an example with $n=8$. 
		We define $M = V\setminus \set*{s,r,t}$. 
		For each $\xinst_i$ ($i \in  [n-4]$), we draw edges $(s, v)$ for $v \in M$ and $(v, t)$ for $v \in \Set*{v' \in M}{v' > i}$, which constitute optimal $s$--$t$ paths of length $2$. 
		In addition, for each $\xinst_i$, we draw edges $(i, r)$ and $(r, t)$, where $s \to i \to r \to t$ is the only suboptimal path of length $3$. 
		Letting $t_i = 2.5$ for $i \in [n-4]$, we prove that $\Ucal$ can shatter those $n-4$ instances, i.e., $A_\rhob$ can return suboptimal solutions to any subset of $\set*{\xinst_1\dots,\xinst_{n-4}}$ by appropriately setting $\rhob$. 
	
		Let $S \subseteq [n-4]$ indicate a subset of instances, to which we will make $A_\rhob$ return suboptimal solutions. 
		We show that for any $S$, we can set $\rhob$ so that $A_\rhob$ returns $s \to i \to r \to t$ to $\xinst_i$ if and only if $i \in S$. 
		We refer to the vertex labeled by $n-3$ as $m$, which we use to ensure that every instance has an optimal path $s \to m \to t$. 
		We set $\rhob$ as follows: 
		$\rho_s = n$ (or an arbitrary value), 
		$\rho_r = \rho_t = 0$, 
		$\rho_i = i+2$ if $i \in S \cup \set*{m}$, 
		and $\rho_i = n$ (or a sufficiently large value) if $i \in [n-4]\setminus S$. 
		If $A_\rhob$ with this $\rhob$ is applied to $\xinst_i$, it iteratively selects vertices in $S \cup \set*{m}$ in increasing order of their labels until a vertex that has a child is selected. 
		Once a vertex with a child is expanded, it ends up returning $s \to i \to r \to t$ if $i \in S$ and $s \to v \to t$ for some $v > i$ if $i \notin S$.
		We below confirm this more precisely, separately for \gbfs and \astar. 
		\paragraph{\gbfs.} 
		We consider applying \gbfs $A_\rhob$ to $\xinst_i$. 
		$A_\rhob$ first expands $s$ and add vertices in $M$ to $\open$. 
		Since vertices in $v \in [n-4]\setminus S$ have sufficiently large scores of $n$, they are never selected before any vertex in $S \cup \set*{m}$. 
		Thus, $A_\rhob$ selects a vertex from $S \cup \set*{m}$ with the smallest score. 
		If the selected vertex, denoted by $v$, satisfies $v < i$, there is no child of $v$; hence, nothing is added to $\open$, and we go back to \Cref{step:gbfs_argmin}. 
		In this way, $A_\rhob$ iteratively moves $v \in S \cup \set*{m}$ that has no child from $\open$ to $\closed$. 
		Consider the first time when the selected vertex $v \in S \cup \set*{m}$ has a child $\vc$ (this situation is guaranteed to occur since $m$ always has a child). 
		If $i \notin S$, we have $v \neq i$ since $v$ is selected from $S \cup \set*{m}$. 
		Then, since $v$'s child is $\vc = t$, $A_\rhob$ returns $s \to v \to t$ with $v \neq i$. 
		If $i \in S$, then $i$ has the smallest score ($\rho_i = i+2$) among all vertices in $S \cup \set*{m}$ that have a child. Thus, $A_\rhob$ selects $i$ and opens $r$. 
		Now, $r$ has the smallest score of $\rho_r = 0$. Therefore, $A_\rhob$ selects $r$ and reaches $t$, returning $s \to i \to r \to t$.  
		Consequently, $A_\rhob$ returns the suboptimal path if and only if $i \in S$. 
		\paragraph{\astar.} 
		It first expands $s$ and add $M$ to $\open$. 
		Since $g_v = 1$ for all $v \in M$, only $\rhob$ values matter when comparing the scores, as with the case of \gbfs. 
		Therefore, \astar iterates to move vertices in $S \cup \set*{m}$ from $\open$ to $\closed$ until a vertex that has a child is selected. 
		Consider the first time a selected vertex $v$ has a child $\vc$ (so far, $s$ is not reopened since $g_s = 0$). 
		As with the case of \gbfs, we have $v\neq i$ and $\vc = t$ if $i \notin S$, or $v = i$ and $\vc = r$ if $i \in S$. 
		Now, every $v' \in \open \setminus \set*{\vc}$ has a score of at least $4$ since $g_{v'} = 1$ and $\rho_{v'} \ge 3$ for $v' \in M$. 
		Therefore, if $i \notin S$, $t \in \open$ has the smallest score of $g_t + \rho_t = 2 + 0 = 2$. 
		Thus, $A_\rhob$ next selects $t$ and returns $s \to v \to t$, where $v \neq i$. 
		If $i \in S$, since $r \in \open$ has the smallest score of $g_r + \rho_r = 2 + 0 = 2$, $A_\rhob$ selects $r$ and opens $t$. 
		Then, since $t$ has the score of $g_t + \rho_t = 3 + 0 = 3$, $A_\rhob$ selects $t$ and returns $s \to i \to r \to t$. 
		To conclude, $A_\rhob$ returns the suboptimal path if and only if $i \in S$.
	\end{proof}
}

\section{Toward better guarantees on the suboptimality of \astar}\label{sec:suboptimality_astar}

Given the results in \Cref{sec:upper_bound,sec:lower_bound}, a major open problem is to close the $\tilde\Ord(n)$ gap\footnote{
	We use $\tilde\Ord$ and $\tilde\Omega$ to hide logarithmic factors of $n$ and $N$ for simplicity.
} 
between the $\Ord(n^2 \lg n)$ upper bound and the $\Omega(n)$ lower bound for \astar in general cases. 
This problem seems very complicated, as we will discuss in \Cref{sec:conclusion}. 
Instead, we here study a particular case where we want to bound the expected suboptimality of \astar, which is an important performance measure since learned heuristic values are not always admissible. 
We show that a general bound obtained from \Cref{thm:astar_upper} can be sometimes improved by using a $\rhob$-dependent worst-case bound \citep{Valenzano2014-sw}. 

For any $\xinst \in \Pi$, let $\opt(\xinst)$ and $\cost_\rhob(\xinst)$ be the costs of an optimal solution and an $s$--$t$ path returned by $A_\rhob$, respectively, and let $u_\rhob(\xinst) = \cost_\rhob(\xinst) - \opt(\xinst)$ be the suboptimality. 
From \Cref{thm:astar_upper} and \Cref{prop:complexity_bound}, we can obtain the following high-probability bound on the expected suboptimality: 
\begin{equation}\label{eq:supopt_bound}
	\mathop{\E}_{\xinst\sim\Dcal} \brc*{\cost_\rhob(\xinst) - \opt(\xinst)} \le \frac{1}{N} \sum_{i=1}^N \prn*{\cost_\rhob(\xinst) - \opt(\xinst)} + \tilde\Ord\prn*{H\sqrt{\frac{n^2 + \lg \frac{1}{\delta}}{N}}}.
\end{equation}
That is, the expected suboptimality can be bounded from above by the empirical suboptimality over the $N$ training instances (an empirical term) plus an $\tilde\Ord(H\sqrt{n^2/N})$ term (a complexity term). 
While this bound is useful when $N \gg n^2$, we may not have enough training instances in practice. 
In such cases, the complexity term becomes dominant and prevents us from obtaining meaningful guarantees. 
In what follows, we present an alternative bound of the form ``an empirical term $+$ a complexity term'' that can strike a better balance between the two terms when $N$ is not large enough relative to $n^2$.  

To this end, we use the notion of \textit{consistency}. 
We say $\rhob$ is \textit{consistent} if $\rho_v \le \rho_{\vc} + w_{(v, \vc)}$ holds for all $(v, \vc) \in E$. 
If $\rhob$ is consistent, \astar without reopening returns an optimal solution. 
\citet[Theorem 4.6]{Valenzano2014-sw} revealed that for any instance $\xinst \in \Pi$, the suboptimality of \astar can be bounded by the inconsistency accumulated along an optimal path (excluding the first edge containing $s$) as follows:\footnote{The original theorem in \citep{Valenzano2014-sw} is applicable only to the case where \astar does not reopen vertices and $\rho_t = 0$. These restrictions are unnecessary as detailed in \Cref{app_sec:valenzano}, and thus our result holds regardless of reopening.}
\begin{equation}\label{eq:inconsistency_valenzano}
	\cost_\rhob(\xinst) - \opt(\xinst) \le \Delta_\rhob(\xinst) \coloneqq \sum_{(v, \vc) \in S^*(\xinst), v \neq s} \max\set*{\rho_v - \rho_{\vc} - w_{(v, \vc)}, 0},
\end{equation}
where  $S^*(\xinst) \subseteq E$ is an optimal solution to $\xinst$ (if there are multiple optimal solutions, we break ties by using the lexicographical order induced from the total order defined in \Cref{assump:tie_break}). 
We call $\Delta_\rhob(\xinst)$ the inconsistency (of $\rhob$ on $S^*(\xinst)$).  

Given $N$ instances $\xinst_1,\dots,\xinst_N$, we can compute the empirical inconsistency, $\frac{1}{N}\sum_{i=1}^N \Delta_\rhob (\xinst_i)$, at the cost of solving the $N$ instances, which we will use as an empirical term. 
To define the corresponding complexity term, we regard $\Delta_\rhob(\cdot):\Pi\to[0,\hat H]$ as an inconsistency function parameterized by $\rhob$, where we will discuss how large $\hat H > 0$ can be later, and we let $\hat \Ucal = \Set*{\Delta_\rhob:\Pi\to [0, \smash{\hat H}]}{\rhob\in\R^n}$. 
The following theorem says that the class $\hat\Ucal$ of inconsistency functions has a smaller pseudo-dimension than the class $\Ucal$ of general utility functions. 

\begin{theorem}\label{thm:pdim_inc}
	For the class $\hat\Ucal$ of inconsistency functions, it holds that $\Pdim(\hat\Ucal) = \Ord(n \lg n)$. 
\end{theorem}

By using \eqref{eq:inconsistency_valenzano}, \Cref{prop:complexity_bound}, and \Cref{thm:pdim_inc}, we can obtain the following high-probability bound on the expected suboptimality, whose complexity term has a better dependence on $n$ than that of \eqref{eq:supopt_bound}: 
\begin{equation}\label{eq:inc_subopt_bound}
	\mathop{\E}_{\xinst\sim\Dcal} \brc*{\cost_\rhob(\xinst) - \opt(\xinst)} 
	\le 
	\mathop{\E}_{\xinst \sim \Dcal}[\Delta_\rhob(\xinst)] 
	\le 
	\frac{1}{N}\sum_{i=1}^N \Delta_\rhob (\xinst_i) 
	+ 
	\tilde\Ord\prn*{\hat H \sqrt{\frac{n + \lg \frac{1}{\delta}}{N}}}. 
\end{equation}
This bound is uniform for all $\rhob \in \R^n$, as with other bounds discussed so far. 
Thus, the bound holds even if we choose $\rhob$ to minimize the empirical inconsistency. 
Note that the empirical inconsistency is convex in $\rhob$ since $\Delta_\rhob (\xinst_i)$ consists of a maximum of a linear function of $\rhob$ and zero, hence easier to minimize than the raw empirical suboptimality in practice (and suitable for a recent online-convex-optimization framework \citep{Khodak2022-sf}). 

Before proving \Cref{thm:pdim_inc}, we present a typical example to show that the inconsistency is not too large relative to the suboptimality. 

\begin{example}\label{exmp:inc_vs_subopt}
	Suppose that every edge weight $w_e$ is bounded to $[0, \ell]$, which ensures that the suboptimality $u_\rhob$ is at most $H = \ell(n-1)$ for any $\rhob \in \R^n$. 
	For simplicity, we consider the following natural way to compute $\rhob$ values: 
	compute an estimate $\hat w_e \in [0, \ell]$ of $w_e$ for each $e\in E$ and let $\rho_v$ be the cost of a shortest $v$--$t$ path with respect to $\set*{\hat w_e}_{e\in E}$. 
	Then, $\rhob$ enjoys the consistency with respect to $\set*{\hat w_e}_{e\in E}$, i.e., $\rho_v \le \rho_{\vc} + \hat w_{(v,\vc)}$ for every $(v, \vc) \in E$. 
	Therefore, it holds that 
	\[
		\Delta_\rhob(\xinst) \le 
		\sum_{(v, \vc) \in S^*(\xinst), v \neq s} \max\set*{\rho_v - \rho_{\vc} - w_{(v,\vc)}, 0}
		\le 
		\sum_{(v, \vc) \in S^*(\xinst), v \neq s} \abs*{\hat w_{(v, \vc)} - w_{(v, \vc)}}. 
	\]
	Hence $\Delta_\rhob$ is at most $\hat H = \ell (n-2)$, implying that $\Delta_\rhob$ does not largely exceed the suboptimality. 
	If empirically accurate estimates $\hat w_e$ for $e \in S^*(\xinst)$ are available, the inconsistency becomes small. 
\end{example}

\shortlongproofs{
	We prove \Cref{thm:pdim_inc} by using the analysis framework by \citet{Balcan2021-jv}. 
	Roughly speaking, if we fix $\xinst\in\Pi$ and see $\Delta_\rhob(\xinst)$ as a function of $\rhob$, it exhibits a piecewise linear structure. 
	By using this structure and \citep[Theorem 3.3]{Balcan2021-jv}, we obtain \Cref{thm:pdim_inc}. 
	We below present a proof sketch due to the space limitation. 
	Please see the supplementary for the complete proof.

	\begin{proof}[Proof sketch]
		Fixing any instance $\xinst \in \Pi$, we define the so-called \textit{dual class} of $\hat\Ucal$ as $\hat\Ucal^*\subseteq {[0, \hat H]}^{\R^n}$, where each $\Delta^*_\xinst \in \hat\Ucal^*$ takes $\rhob\in \R^n$ as input and returns $\Delta_\rhob(\xinst) \in [0, \hat H]$.
		Once $\xinst$ is fixed, $S^*(\xinst)$ is unique due to the tie-breaking. 
		Thus, $\Delta^*_\xinst(\rhob) = \sum_{(v, \vc) \in S^*(\xinst), v \neq s} \max\set*{\rho_v - \rho_{\vc} - w_{(v, \vc)}, 0}$ is uniquely defined as a piecewise linear function of $\rhob$, where pieces are specified by $|E| = \Ord(n^2)$ halfspace boundary functions: $\bdf^{(v, \vc)}(\rhob) = \Ibb\prn*{\rho_v - \rho_{\vc} - w_{(v, \vc)} > 0}$ for all edges $(v, \vc) \in E$. 
		Let $\Fcal$ and $\Bcal$ be the classes of linear and halfspace functions of $\rhob$, respectively. 
		From \citep[Theorem 3.3]{Balcan2021-jv}, 
		\[
			\Pdim\prn{\hat\Ucal} = \Ord\prn*{ \prn*{\Pdim(\Fcal^*) + \VCdim(\Bcal^*)} \lg \prn*{\Pdim(\Fcal^*) + \VCdim(\Bcal^*)} + \VCdim(\Bcal^*)\lg\nbd}
		\]
		holds, where $\Fcal^*$ and $\Bcal^*$ are the dual classes of $\Fcal$ and $\Bcal$, respectively, and $\nbd$ is the number of boundary functions required to specify all the pieces. 
		As mentioned above, $\nbd = \Ord(n^2)$ holds. 
		Furthermore, we can regard $\Fcal^*$ and $\Bcal^*$ as classes of linear and halfspace functions on $\R^{n+1}$, respectively, hence $\Pdim(\Fcal^*) = \VCdim(\Bcal^*) = n + 1$. 
		Thus, we obtain $\Pdim(\hat\Ucal) = \Ord(n \lg n)$.
	\end{proof}
}{
	We prove \Cref{thm:pdim_inc} by using the general analysis framework by \citet{Balcan2021-jv}. 
	To begin with, we introduce Assouad's dual class, which provides a formal definition of the class of functions of $\rhob$. 

	\begin{definition}[\citet{Assouad1983-vb}]\label{def:dual}
		Given a class, $\Hcal\subseteq \R^\Ycal$, of functions $h:\Ycal\to\R$,
		the \textit{dual class} of $\Hcal$ is defined as
		$\Hcal^* = \Set*{h^*_y:\Hcal \to \R}{y\in\Ycal}$ such that $h^*_y(h) = h(y)$ for each $y \in \Ycal$.
	\end{definition}

	In our case, we have $\Ycal = \Pi$ and $\Hcal = \hat\Ucal$, 
	and each $\Delta^*_\xinst \in \hat\Ucal^*$ is associated with an instance $\xinst \in \Pi$. 
	The following definition will be used to capture the piecewise structure of the dual class $\hat\Ucal^*$. 

	\begin{definition}[{\citet[Definition 3.2]{Balcan2021-jv}}]\label{def:piecewise_decomposable}
		A class, $\Hcal\subseteq \R^\Ycal$, of functions is
		\textit{$(\Fcal, \Bcal, \nbd)$-piecewise decomposable}
		for a class $\Bcal\subseteq\set*{0,1}^\Ycal$ of boundary functions and a class $\Fcal\subseteq \R^{\Ycal}$ of piece functions if the following condition holds:
		for every $h\in\Hcal$, there exist $\nbd$ boundary functions $\bdf^{(1)},\dots,\bdf^{(\nbd)} \in\Bcal$ and a piece function $\pf_\bb$ for each binary vector $\bb\in\set*{0,1}^\nbd$ such that for all $y \in\Ycal$,
		it holds that $h(y) = \pf_{\bb_y}(y)$ where $\bb_y = (\bdf^{(1)}(y), \dots, \bdf^{(\nbd)}(y)) \in \set*{0,1}^\nbd$.
	\end{definition}

	The following result of \citep{Balcan2021-jv} provides an upper bound on the pseudo-dimension of a class of functions via the piecewise structure of the dual class. 

	\begin{proposition}[{\citet[Theorem 3.3]{Balcan2021-jv}}]\label{prop:balcan_pdim_bound}
	  Let $\Ucal\subseteq \R^\Pi$ be a class of functions. 
		If $\Ucal^* \subseteq \R^\Ucal$ is $(\Fcal, \Bcal, \nbd)$-piecewise decomposable with a class $\Bcal\subseteq \set*{0,1}^\Ucal$ of boundary functions and a class $\Fcal \subseteq \R^\Ucal$ of piece functions, the pseudo-dimension of $\Ucal$ is bounded as follows: 
		\[
			\Pdim\prn{\Ucal} = \Ord\prn*{ \prn*{\Pdim(\Fcal^*) + \VCdim(\Bcal^*)} \lg \prn*{\Pdim(\Fcal^*) + \VCdim(\Bcal^*)} + \VCdim(\Bcal^*)\lg\nbd}.
		\]
	\end{proposition}

	Now, we are ready to prove \Cref{thm:pdim_inc}. 

	\begin{proof}[Proof of \Cref{thm:pdim_inc}]
		Since there is a one-to-one correspondence between $\Delta_\rhob \in \hat\Ucal$ and $\rhob\in \R^n$, we below identify $\Delta_\rhob$ with $\rhob$ for simplicity. 
		Let $\Bcal = \Set*{\Ibb\prn*{\iprod{\zb, \rhob} + z_0}}{(z_0, \zb) \in \R^{n+1}} \subseteq \set*{0, 1}^{\hat\Ucal}$ and $\Fcal = \Set*{\iprod{\zb, \rhob} + z_0}{(z_0, \zb) \in \R^{n+1}} \subseteq \R^{\hat\Ucal}$ be classes of boundary and piece functions, respectively. 
		We show that $\hat\Ucal^*$ is $(\Fcal, \Bcal, \Ord(n^2))$-piecewise decomposable. 

		Fix any $\Delta^*_\xinst \in \hat\Ucal^*$; this also uniquely specifies an instance $\xinst \in \Pi$ and an optimal solution $S^*(\xinst) \subseteq E$ (due to the tie-breaking). 
		We consider $\nbd = |E| = \Ord(n^2)$ boundary functions of form $\bdf^{(v, \vc)}(\rhob) = \Ibb\prn*{\rho_v - \rho_{\vc} - w_{(v, \vc)} > 0}$ for all edges $(v, \vc) \in E$. 
		We below confirm that these boundary functions partition $\R^n \ni \rhob$ into some regions so that in each region, $\Delta^*_\xinst(\rhob)$ can be written as a linear function of $\rhob$, which belongs to $\Fcal$. 
		For each binary vector $\bb_{\rhob} = \prn*{\bdf^{(v, \vc)}(\rhob)}_{(v, \vc)\in E} \in \set*{0, 1}^\nbd$, we define a subset $S_{\rhob}(\xinst)$ of $S^*(\xinst)$ as $S_{\rhob}(\xinst) = \Set*{(v, \vc) \in S^*(\xinst)}{\bdf^{(v, \vc)}(\rhob) = 1, v \neq s}$; that is, each $(v, \vc) \in S_{\rhob}(\xinst)$ satisfies $v \neq s$ and $\rho_v - \rho_{\vc} - w_{(v, \vc)} > 0$. 
		From the definition of $\Delta_\rhob(\xinst)$, we have $\Delta^*_\xinst(\rhob) = \Delta_\rhob(\xinst) = \sum_{(v, \vc) \in S_{\rhob}(\xinst)}(\rho_v - \rho_{\vc} - w_{(v, \vc)})$.  
		This is a linear function of $\rhob$, and thus we can choose a piece function $\pf_{\bb_{\rhob}} \in \Fcal$ such that $\Delta^*_\xinst = \pf_{\bb_{\rhob}}$. 
		This relation holds for every $\bb_\rhob \in \set*{0,1}^K$, and thus we have $\Delta^*_\xinst(\rhob) = \pf_{\bb_{\rhob}}(\rhob)$ for all $\rhob \in \R^n$. 
		Hence $\hat\Ucal^*$ is $(\Fcal, \Bcal, \Ord(n^2))$-piecewise decomposable. 

		Since $\Fcal^*$ and $\Bcal^*$ can be seen as classes of linear and halfspace functions of $(z_0, \zb) \in \R^{n+1}$, respectively, we have $\Pdim(\Fcal^*) = \VCdim(\Bcal^*) = n+1$ (see also \citep[Lemma 3.10]{Balcan2019-rr}, a preprint version of~\citep{Balcan2021-jv}). 
		Therefore, from \Cref{prop:balcan_pdim_bound}, we obtain $\Pdim(\hat\Ucal) = \Ord(n \lg n)$.
	\end{proof}
}

\section{Conclusion and discussion}\label{sec:conclusion}
We have studied the sample complexity of learning heuristic functions for \gbfs and \astar on graphs with a fixed vertex set of size $n$. 
The crucial step is to evaluate the pseudo-dimension of the class of utility functions. 
For \gbfs and \astar, we have proved that the pseudo-dimensions are upper bounded by $\Ord(n\lg n)$ and $\Ord(n^2\lg n)$, respectively. 
As for \astar, we have shown that the bound can be improved to $\Ord(n^2\lg d)$ if every vertex has a degree of at most $d$ and to $\Ord(n \lg n)$ if edge weights are bounded integers. 
We have also presented the $\Omega(n)$ lower bounds for \gbfs and \astar, implying that our bounds for \gbfs and \astar under the integer-weight condition are tight up to a $\lg n$ factor. 
Finally, we have discussed bounds on the suboptimality of \astar and obtained a guarantee with a better complexity term by evaluating the pseudo-dimension of the class of inconsistency functions. 

As mentioned in \Cref{sec:suboptimality_astar}, an open problem is to close the gap between the upper and lower bounds regarding \astar for general cases. 
This, however, does not seem straightforward. We here discuss the reasons for the difficulty. 
As regards the upper bound, the bottleneck is the bound of $(n-2)!$ on $|\Gcal_v|$, but this cannot be improved in general, as shown in \Cref{exmp:large_Gv}. 
Taking this into account, the direct use of Sauer's lemma would not yield better upper bounds. 
Thus, we need to use some special structures of the hyperplanes (e.g., each has only two variables), which would require more complicated analysis.  
As for the lower bound, the construction of the $\Omega(n)$ instances in \Cref{thm:lower} relies on the fact that $\rhob$ has an $n$ degree of freedom. 
In addition, \Cref{thm:astar_integer} implies that we need to consider instances with non-integer edge weights (or exponentially large integer weights in $n$) to obtain a lower bound of $\tilde\Omega(n^2)$.  
Considering the above, we would have to use involved techniques for constructing a set of $\tilde\Omega(n^2)$ shatterable instances. 
Another interesting future direction is to improve upper bounds on the pseudo-dimension by restricting heuristic functions to some classes, as mentioned in \Cref{sec:problem_formulation}. 
We finally discuss limitations of our work. 
As mentioned in \Cref{sec:problem_formulation}, we require every instance to be defined on (subsets of) a fixed vertex set. 
Also, our work does not cover the case where heuristic function values can change depending on instance-dependent features. 
Studying how to overcome these limitations would also constitute interesting future work.

\subsection*{Acknowledgements}
This work was supported by JST ERATO Grant Number JPMJER1903 and JSPS KAKENHI Grant Number JP22K17853.

\bibliographystyle{abbrvnat}
\bibliography{PaperpileHeuristicSearch}

\appendix
\setcounter{equation}{0}
\renewcommand{\theequation}{A\arabic{equation}}

\clearpage

\begin{center}
	{\fontsize{18pt}{0pt}\selectfont \bf Appendix}
\end{center}

\newcommand{\Popt}{P_{\mathrm{opt}}}
\newcommand{\inc}{\mathrm{Inc}}
\newcommand{\dg}{{\delta g}}
\newcommand{\gtv}[2]{g^{(#1)}_{#2}}
\newcommand{\dgtv}[2]{\delta g^{(#1)}_{#2}}
\newcommand{\selected}{\texttt{SELECTED}}

\section{Comparisons with previous results on tree search}\label{app_sec:comparison}

We compare our upper bounds with those of existing results on general tree search \citep{Balcan2021-kv,Balcan2021-fz}. 

\citet{Balcan2021-kv} studied the pseudo-dimension for tree-search algorithms in the following situation; a tree-search algorithm with $d$ configurable parameters builds a search tree of size at most $\kappa$ by iteratively choosing an action from a set of at most $T$ possible actions. 
In this setting, they obtained an $\Ord(d\kappa \log T + d \log d)$ upper bound. 
\citet{Balcan2021-fz} removed the dependence on $\kappa$ assuming scores governing the tree search to be defined by \textit{path-wise} functions. 
Their bound is $\Ord(d \Delta^2\log k + d \Delta \log T)$, where $\Delta$ and $k$ are the maximum depth and the number of children, respectively, of search trees. 
Since $\kappa$ can be exponential in the depth $\Delta$, this is a considerable improvement in this context. 

In our setting, since there are $n$ configurable parameters, $\rhob\in\R^n$, we have $d = n$. 
If we apply the bound of~\citep{Balcan2021-kv} to \gbfs/\astar regarded as a tree-search algorithm that iteratively builds search states, $\kappa$ and $T$ values can be as large as $\Omega(n)$. 
This is because \gbfs/\astar can perform $\Omega(n)$ iterations, where each iteration increases the tree size, and the number of possible actions is equal to the size of $\open$, which is $\Omega(n)$ in general. 
Moreover, for \astar with reopening, the number of iterations can be $\Omega(2^n)$ as mentioned in \Cref{subsec:astar_upper}, implying $\kappa = \Omega(2^n)$. 
Thus, only in the case of \astar without reopening, the previous bound matches our $\Ord(n^2\lg n)$ bound (\Cref{thm:astar_upper}). 
As for \gbfs and \astar with reopening, our \Cref{thm:gbfs_upper,thm:astar_upper} provide $\Ord(n\lg n)$ and $\Ord(n^2\lg n)$ bounds, respectively, which improve the $\Ord(n^2 \lg n)$ and $\Ord(n2^n \lg n)$ bounds, respectively, implied by the previous result. 

As for the result of~\citep{Balcan2021-fz}, seeing \gbfs/\astar as a tree-search algorithm again, the maximum tree depth is as large as the number of vertices in general, i.e., $\Delta = \Omega(n)$. 
Also, the number of children can be as large as the size of $\open$, hence $k = \Omega(n)$. 
Thus, the result of~\citep{Balcan2021-fz} leads to an $\Ord(n^3\lg n)$ bound, which is larger than any of our upper bounds.

\section{How to deal with varying $t$}\label{app_sec:additional}
As mentioned in \Cref{subsec:problem}, given a fixed $t \in V$, each entry $\rho_v$ in $\rhob$ indicates an estimated cost of the shortest $v$--$t$ path. 
Therefore, if $t$ changes over instances, we need to define $\rhob$ for each $t$, which we here denote by $\rhob^t \in \R^n$. 
If $t$ changes, the structure of path-finding instances also greatly changes. 
Thus, it is natural to evaluate the performance of \gbfs/\astar separately for each $t$. 
Specifically, we take $\Dcal$ to be a conditional distribution, from which path-finding instances with fixed $t$ are drawn, and we analyze the sample complexity for each possible $t \in V$. 
In practice, $\set*{\rhob^t}_{t \in V}$ may be obtained by, e.g., learning a function that embeds vertices into a metric space and measuring distances on the space, as mentioned in \Cref{subsec:problem}. 
In this case, an embedding function with tunable parameters governs all the $\set*{\rhob^t}_{t \in V}$ values. 
Considering such situations, we need to bound the sample complexity of learning heuristic functions for all possible $t \in V$. 
This can be done at a slight cost of increasing the bound in \Cref{prop:complexity_bound} by taking a union bound over all possible $t \in V$. 
(Note that the upper bounds on the pseudo-dimension in \Cref{thm:gbfs_upper,thm:astar_upper,thm:pdim_inc} hold separately for each $t \in V$ by regarding $\rhob$ as $\rhob^t$.) 
Since there are at most $n$ possible choices of $t$, this replaces $\delta$ in \Cref{prop:complexity_bound} with $\delta/n$, yielding only a $\lg n$ additive factor. 
This effect is small relative to that of the pseudo-dimension term.

\section{A worst-case analysis of \astar regardless of reopening}\label{app_sec:valenzano} 
We show that the existing bound on the suboptimality of \astar by \citet{Valenzano2014-sw} holds regardless of whether we allow \astar to reopen vertices or not, 
and we also remove a minor assumption of $\rho_t = 0$.  
Note that the result of \citep{Valenzano2014-sw}, which focuses on the case without reopening, does not immediately imply the same bound for \astar with reopening since reopening sometimes degrades the solution quality \citep{Sepetnitsky2016-zo}.

We fix an instance and define the inconsistency of an edge $(v, v') \in E$ as 
\begin{align}\label{eq:inc_def}
	\inc(v, v') = \max\set*{\rho_{v} - \rho_{v'} - w_{(v, v')}, 0}.
\end{align}
Fix an optimal $s$--$t$ path and let $\Popt = v_0, v_1, \dots, v_k$ be a sequence of vertices on the optimal path, 
where $v_0 = s$, $v_k = t$, and the optimal $s$--$t$ path visits the vertices in this order. 
Suppose that $v_k = t$ is first selected at the ($T+1$)st iteration, at which the algorithm terminates. 

\begin{theorem}\label{thm:valenzano_reopen}
	Let $\cost$ be the cost of an $s$--$t$ path returned by \astar (\Cref{alg:astar}) with/without reopening, and let $\opt$ be the cost of $\Popt = v_0, v_1, \dots, v_k$. 
	It holds that 
	\[
		\cost \le \opt + \sum_{j = 1}^{k-1} \inc(v_j, v_{j+1}).
	\] 
\end{theorem}

The theorem was proved by \citet{Valenzano2014-sw} for \Cref{alg:astar} without reopening. 
Their proof uses the fact that once a vertex is added to $\closed$, it never gets out of $\closed$. 
If we allow \astar to reopen vertices, the fact is not always true. 
Therefore, we need to prove the theorem without using the property of $\closed$. 
To this end, we define lists of selected vertices, which play a similar role to $\closed$ in our proof. 
Formally, for $\tau = 0,1,\dots, T+1$, 
we define $\selected_\tau$ as a list of vertices that have been selected in \Cref{step:astar_argmin} in the first $\tau$ iterations. 
Note that even with reopening, once a vertex is added to $\selected_\tau$, it never gets out of the list. 

As in \citep{Valenzano2014-sw}, we derive the bound in \Cref{thm:valenzano_reopen} by decomposing $\Popt$ into two subpaths, which are defined based on the following \textit{shallowest} vertex. 
 
\begin{lemma}\label{lem:shallowest} 
	We say a vertex $v_i \in \Popt = v_0, v_1, \dots, v_k$ is the shallowest vertex at $\tau \in \set*{0,1,\dots,T}$ if it satisfies the following conditions after the $\tau$th iteration: 
	\begin{align} 
		v_i \in \emph{\open}\setminus \emph{\selected}_\tau & & 
		\text{and} & & 
		\Set*{v_j \in \Popt}{j < i} \subseteq \emph{\selected}_\tau. 
	\end{align} 
	For every $\tau = 0,1,\dots,T$, a shallowest vertex always exists. 
\end{lemma}

\begin{proof}
	We prove the claim by induction. 
	If $\tau = 0$, $v_0$ is the shallowest since we have $\selected_0 = \emptyset$ and $\open = \set*{v_0}$. 
	If $\tau = 1$, $v_1$ is the shallowest vertex since we have $\selected_1 = \set*{v_0}$ and $v_1 \in \open \setminus \selected_1$. 
	Assume that the claim is true for $\tau'<\tau$ and let $v_{i'} \in \open \setminus \selected_{\tau-1}$ be the shallowest vertex at $\tau-1$. 
	We consider two cases: $v_{i'}$ or $v \neq v_{i'}$ is selected at the $\tau$th iteration. 
	If $v\neq v_{i'}$ is selected, we have $\selected_\tau = \selected_{\tau-1} \cup \set*{v}$ and $v_{i'} \in \open \setminus \selected_\tau$; thus $v_{i'}$ remains the shallowest at $\tau$. 
	If $v_{i'}$ is selected, take the longest subpath of $\Popt$ that starts from $v_{i'}$ and is contained in $\selected_\tau$.  
	We denote such a subpath by $v_{i'},\dots, v_{i''}$, where $i'' < k$ holds since $v_k$ is never selected in the first $T$ iterations. 
	From the definition of the subpath, we have $v_{i''+1} \notin \selected_\tau$. 
	Moreover, $v_{i''+1}$ must have been opened since its parent $v_{i''}$ has been selected. 
	Thus, $v_{i''+1} \in \open \setminus \selected_\tau$ holds. 
	Furthermore, we have $\set*{v_0,\dots,v_{i'-1}} \subseteq \selected_\tau$ due to the induction hypothesis and $\set*{v_{i'},\dots, v_{i''}} \subseteq \selected_\tau$ from the definition of the subpath. 
	Thus, $v_{i''+1}$ is the shallowest vertex at $\tau$. 
	To conclude, the shallowest vertex at $\tau$ exists in any case. 
	The proof is completed by induction. 
\end{proof}

\subsection{Decomposing the suboptimality term with the shallowest vertex}

For every $v \in V$, we let $g^*_v$ and $\rho^*_v$ denote the costs of optimal $s$--$v$ and $v$--$t$ paths, respectively. 
We use $g^*_{v, v'}$ to denote the cost of an optimal $v$--$v'$ path for any pair of $v, v' \in V$; it holds that $g^*_v = g^*_{s, v}$. 
For each $v \in V$, let $\gtv{\tau}{v}$ be the $g_v$ value after the $\tau$th iteration, where $\gtv{0}{s} = 0$ and $\gtv{0}{v} = \infty$ for $v\neq s$. 
If $g_v$ is updated in the ($\tau+1$)st iteration, we have $\gtv{\tau+1}{v} < \gtv{\tau}{v}$; otherwise we have $\gtv{\tau+1}{v} = \gtv{\tau}{v}$. 
Thus, $\gtv{\tau}{v}$ is non-increasing in $\tau$.   
We define $\dgtv{\tau}{v} = \gtv{\tau}{v} - g^*_v$ as the $g$-cost error of $v$ after the $\tau$th iteration, which is also non-increasing in $\tau$. 

The following lemma states that the suboptimality can be decomposed into two parts: 
a $g$-cost error of the shallowest vertex $v_i$ and the inadmissibility of $\rho_{v_i}$ (subtracted by $\rho_t$). 

\begin{lemma}\label{lem:decompose_subopt}
	If $v_i$ is the shallowest vertex at $T$, it holds that 
	\[
		\cost \le \opt + \dgtv{T}{v_i} + \rho_{v_i} - \rho^*_{v_i} - \rho_t.
	\] 
\end{lemma}

\begin{proof}
	After the $T$th iteration, the score of $v_i$ is 
	\begin{align}
		\gtv{T}{v_i} + \rho_{v_i} 
		&= g^*_{v_i} + \dgtv{T}{v_i} +\rho_{v_i} \\
		&= g^*_{v_i} + \rho^*_{v_i} + \dgtv{T}{v_i} +\rho_{v_i} - \rho^*_{v_i} \\
		&= \opt + \dgtv{T}{v_i} +\rho_{v_i} - \rho^*_{v_i}.
	\end{align}
	Since $v_k = t$ is selected at the ($T+1$)st iteration instead of $v_i$, it holds that $\gtv{T}{t} + \rho_{t}\le\gtv{T}{v_i} + \rho_{v_i}$. 
	Since we have $\cost \le \gtv{T}{t}$, we obtain the statement by rearranging the terms. 
\end{proof}

\subsection{Bounding $\rho_{v_i} - \rho^*_{v_i} - \rho_t$}
 
We prove a general lemma for later use, which implies an upper bound on $\rho_{v_i} - \rho^*_{v_i} - \rho_t$.  
\begin{lemma}\label{lem:ina_inc_bound}
	Let $P = v_0,v_1,\dots,v_k$ be any optimal $v_0$--$v_k$ path. It holds that
	\[
		\rho_{v_0} - \rho_{v_k} - g^*_{v_0, v_k} \le \sum_{i = 0}^{k-1} \inc(v_i, v_{i+1}).
	\]
	
\end{lemma}

\begin{proof}
	From the definition \eqref{eq:inc_def}, $\inc(v_i, v_{i+1}) \ge \rho_{v_i} - \rho_{v_{i+1}} - w_{(v_i, v_{i+1})}$ holds. 
	Therefore, we have 
	\[
		\sum_{i=0}^{k-1} \prn*{\rho_{v_i} - \rho_{v_{i+1}} - w_{(v_i, v_{i+1})}} \le \sum_{i=0}^{k-1} \inc(v_i, v_{i+1}). 
	\]
	Using a telescoping sum argument, we obtain 
	\[
		 \rho_{v_0} - \rho_{v_{k}} - \sum_{i=0}^{k-1} w_{(v_i, v_{i+1})} \le \sum_{i=0}^{k-1} \inc(v_i, v_{i+1}).
	\]
	Since $P$ is optimal, we have $\sum_{i=0}^{k-1} w_{(v_i, v_{i+1})} = g^*_{v_0, v_k}$, thus completing the proof. 
\end{proof}
 
Consider applying \Cref{lem:ina_inc_bound} to the subpath $P=v_i,\dots,v_k$ of $\Popt$, which is an optimal $v_i$--$v_k$ path. 
Since $g^*_{v_i, v_k} = \rho^*_{v_i}$ and $\rho_{v_k} = \rho_t$, it holds that 
\begin{equation}\label{eq:inad_bound_cor}
	\rho_{v_i} - \rho^*_{v_i} - \rho_t \le \sum_{j = i}^{k-1} \inc(v_j, v_{j+1}).   
\end{equation}
Thus, we can obtain an upper bound on $\rho_{v_i} - \rho^*_{v_i} - \rho_t$. 

\subsection{Bounding $\dgtv{T}{v_i}$}

Our goal is to prove the following lemma. 
\begin{lemma}\label{lem:g_cost_error_bound_T}
	Let $\Popt = v_0,\dots,v_k$ be the optimal $s$--$t$ path in the statement of \Cref{thm:valenzano_reopen}. 
	Then, the shallowest vertex $v_i \in \Popt$ at $T$ satisfies 
	\[
		\dgtv{T}{v_i} \le \sum_{j=1}^{i-1} \inc(v_j, v_{j+1}).
	\]	
\end{lemma}
To prove \Cref{lem:g_cost_error_bound_T}, we need the following two lemmas.

\begin{lemma}\label{lem:g_cost_error_bound_induction}
	For $\Popt$ in \Cref{lem:g_cost_error_bound_T} and $i\ge 1$, if $v_{i-1} \in \Popt$ is first selected at the $\tau'$th iteration and satisfies $\dgtv{\tau'}{v_{i-1}} \le \sum_{j=1}^{i-2} \inc(v_j, v_{j+1})$, then $v_i \in \Popt$ satisfies $\dgtv{\tau}{v_i} \le \sum_{j=1}^{i-1} \inc(v_j, v_{j+1})$ for all $\tau = \tau',\dots,T$.
\end{lemma}

\begin{proof}
	In the $\tau'$th iteration, we update $g_{v_i}$ if it is larger than $g_{\mathrm{new}} = \gtv{\tau'}{v_{i-1}} + w_{(v_{i-1}, v_i)}$, hence $\gtv{\tau'}{v_{i}} \le \gtv{\tau'}{v_{i-1}} + w_{(v_{i-1}, v_i)}$. 
	Since $(v_{i-1}, v_i)$ is an edge on $\Popt$, $g^*_{v_i} = g^*_{v_{i-1}} + w_{(v_{i-1}, v_i)}$ holds. Therefore, it holds that $\dgtv{\tau'}{v_{i}} \le \dgtv{\tau'}{v_{i-1}}$. 
	Since $\dgtv{\tau}{v_i}$ is non-increasing in $\tau$, we have $\dgtv{\tau}{v_i} \le \dgtv{\tau'}{v_i}$. 
	If $i = 1$, since $v_0 = s$, we obtain 
	\[
		\dgtv{\tau}{v_1}\le \dgtv{\tau'}{v_1} \le \dgtv{\tau'}{v_0} = \gtv{\tau'}{v_0} - g^*_{v_0} = 0 - 0 = 0.
	\] 
	If $i > 1$, we have 
	\[
		\dgtv{\tau}{v_i} \le \dgtv{\tau'}{v_i} \le \dgtv{\tau'}{v_{i-1}} \le \sum_{j=1}^{i-2} \inc(v_j, v_{j+1}) \le \sum_{j=1}^{i-1} \inc(v_j, v_{j+1}), 
	\]
	where we used the assumption on $\dgtv{\tau'}{v_{i-1}}$ and $\inc(v_{i-1}, v_i) \ge 0$. 
\end{proof}

\begin{lemma}\label{lem:g_cost_error_bound}
	For $\Popt$ in \Cref{lem:g_cost_error_bound_T} and any $\tau = 0,1,\dots, T$, every $v_i \in \Popt \cap \emph{\selected}_\tau$ satisfies $\dgtv{\tau}{v_i} \le \sum_{j=1}^{i-1} \inc(v_j, v_{j+1})$.	
\end{lemma}

\begin{proof}
	The proof is by induction on $\tau$. 
	If $\tau = 0$, the claim is vacuously true since $\selected_\tau = \emptyset$. 
	If $\tau = 1$, only $v_0$ is in $\selected_1$. 
	Since $\gtv{1}{v_0} = g^*_{v_0} = 0$, we have $\dgtv{1}{v_0} = 0$. 
	Thus, the claim is true. 

	Assume that the claim is true after the first $\tau\ge1$ iterations; 
	in other words, 
	for any $\tau'\le \tau$, 
	every $v_{i'} \in \Popt \cap \selected_{\tau'}$ satisfies $\dgtv{\tau'}{v_{i'}} \le \sum_{j=1}^{i'-1} \inc(v_j, v_{j+1})$.  
	Since $\dgtv{\tau}{v_{i'}}$ is non-increasing in $\tau$, from the induction hypothesis, vertices in $\Popt \cap \selected_\tau$ remain to satisfy the inequality after the ($\tau+1$)st iteration. 
	Therefore, we focus on the vertex selected at the ($\tau+1$)st iteration, which is the only new vertex in $\selected_{\tau+1}$. 
	If the selected vertex is not in $\Popt$, the statement is true after the ($\tau+1$)st iteration.
	We below discuss the case where the selected vertex is in $\Popt$. 
	We let $v_i \in \Popt$ be the selected vertex, where $i \ge 1$, and discuss two cases: 
	$v_i$'s parent, $v_{i-1}$, is in $\selected_{\tau}$ or not.  

	\paragraph{Case 1: $v_{i-1} \in \selected_{\tau}$.}
	Suppose that $v_{i-1}$ has been first selected at the $\tau'$th iteration ($\tau' \le \tau$). 
	The induction hypothesis implies $\dgtv{\tau'}{v_{i-1}} \le \sum_{j=1}^{i-2} \inc(v_j, v_{j+1})$. 
	Thus, from \Cref{lem:g_cost_error_bound_induction}, we obtain $\dgtv{\tau+1}{v_i} \le \sum_{j=1}^{i-1} \inc(v_j, v_{j+1})$. 
	
	\paragraph{Case 2: $v_{i-1} \notin \selected_{\tau}$.} 
	In this case, we have $i > 1$ since $v_0$ is selected at the first iteration. 
	Let $v_{i'}$ be the shallowest vertex at $\tau$. 
	Since $v_{j} \in \selected_\tau$ must hold for all $j < i'$, we have $i' < i$. 
	Since $v_i$ is selected instead of $v_{i'}$, it holds that $\gtv{\tau}{v_i} + \rho_{v_i} \le \gtv{\tau}{v_{i'}} + \rho_{v_{i'}}$. Therefore, we have 
	\[
		\dgtv{\tau}{v_i} 
		\le 
		\gtv{\tau}{v_{i'}} - g^*_{v_i} + \rho_{v_{i'}} - \rho_{v_i}
		= 
		\dgtv{\tau}{v_{i'}} + g^*_{v_{i'}} - g^*_{v_i} + \rho_{v_{i'}} - \rho_{v_i}
		= 
		\dgtv{\tau}{v_{i'}} + \rho_{v_{i'}} - \rho_{v_i} - g^*_{v_{i'}, v_i}.
	\]
	We below consider bounding the right-hand side. 
	First, we discuss a bound on $\dgtv{\tau}{v_{i'}}$. 
	Suppose that $v_{i'-1} \in \selected_\tau$ is first selected at the $\tau'$th iteration, where $\tau' \le \tau$. 
	From the induction hypothesis, we have $\dgtv{\tau'}{v_{i'-1}} \le \sum_{j=1}^{i'-2} \inc(v_j, v_{j+1})$. 
	Therefore, \Cref{lem:g_cost_error_bound_induction} implies $\dgtv{\tau}{v_{i'}} \le \sum_{j=1}^{i'-1} \inc(v_j, v_{j+1})$. 
	Next, from \Cref{lem:ina_inc_bound}, we have $\rho_{v_{i'}} - \rho_{v_i} - g^*_{v_{i'}, v_i} \le \sum_{j = i'}^{i-1} \inc(v_j, v_{j+1})$. 
	Consequently, we obtain 
	\[
		\dgtv{\tau+1}{v_i} \le \dgtv{\tau}{v_i} 
		\le 
		\sum_{j=1}^{i'-1} \inc(v_j, v_{j+1}) + \sum_{j = i'}^{i-1} \inc(v_j, v_{j+1}) 
		= 
		\sum_{j = 1}^{i-1} \inc(v_j, v_{j+1}). 
	\]	
	To conclude, every $v_i \in \Popt \cap \selected_{\tau+1}$ satisfies $\dgtv{\tau+1}{v_{i}} \le \sum_{j=1}^{i-1} \inc(v_j, v_{j+1})$. 
	Therefore, the statement is true by induction. 
\end{proof}

Now, we are ready to prove \Cref{lem:g_cost_error_bound_T}. 
\begin{proof}[Proof of \Cref{lem:g_cost_error_bound_T}]
	Since $v_i$ is the shallowest at $T$, $v_{i-1}$ has been first selected at some $\tau$th iteration, where $\tau \le T$, 
	i.e., $v_{i-1} \in \Popt \cap \selected_\tau$. 
	Thus, \Cref{lem:g_cost_error_bound} implies  
	$\dgtv{\tau}{v_{i-1}} \le \sum_{j=1}^{i-2} \inc(v_j, v_{j+1})$. 
	Therefore, from \Cref{lem:g_cost_error_bound_induction}, we obtain 
	$\dgtv{T}{v_{i}} \le \sum_{j=1}^{i-1} \inc(v_j, v_{j+1})$. 
\end{proof}

\subsection{Proof of \texorpdfstring{\Cref{thm:valenzano_reopen}}{Theorem \ref{thm:valenzano_reopen}}}

By summing up the above lemmas and equations, we prove \Cref{thm:valenzano_reopen}. 

\begin{proof}[Proof of \Cref{thm:valenzano_reopen}]
	By using \Cref{lem:decompose_subopt,lem:g_cost_error_bound_T} and \cref{eq:inad_bound_cor}, we obtain 
	\begin{align}
		\cost 
		&\le 
		\opt + \dgtv{T}{v_i} + \rho_{v_i} - \rho^*_{v_i} - \rho_t
		\\
		&\le 
		\opt + \sum_{j=1}^{i-1} \inc(v_j, v_{j+1}) + \sum_{j = i}^{k-1} \inc(v_j, v_{j+1}) 
		\\
		&= 
		\opt + \sum_{j = 1}^{k-1} \inc(v_j, v_{j+1}).\qedhere
	\end{align} 
\end{proof}

\end{document}

%% file: package.tex
\usepackage{amsmath,amsthm,amssymb,mathtools}

\usepackage{algorithm} 

\usepackage[hidelinks, hypertexnames=false, colorlinks=true, citecolor=Navy, linkcolor=Maroon, urlcolor=Orchid, bookmarksnumbered, unicode]{hyperref}

\usepackage[noend]{algpseudocode}

\algnewcommand{\Break}{\textbf{break}}

\usepackage{cleveref}
\crefname{equation}{eq.}{eqs.}                
\Crefname{equation}{Eq.}{Eqs.}
\crefname{step}{Step}{Steps}
\Crefname{step}{Step}{Steps}
\Crefname{lem}{Lemma}{Lemmas}

\usepackage{autonum}

\usepackage[utf8]{inputenc} 
\usepackage[T1]{fontenc}    
\usepackage{amsfonts}       
\usepackage{nicefrac}       
\usepackage{microtype}      
\usepackage{bm}             
\usepackage{bbm}

\usepackage{graphicx}
\usepackage[svgnames]{xcolor} 
\usepackage{subcaption}
\usepackage{wrapfig}

\usepackage{tabularx}       
\usepackage{booktabs}       

\usepackage[numbers, compress]{natbib}

\usepackage{thmtools}
\usepackage{thm-restate}
\usepackage{apxproof}
\usepackage{url}            
\usepackage{braket}
\usepackage{multirow}
\usepackage{lipsum}
\allowdisplaybreaks
\usepackage{lscape}
\usepackage{tcolorbox}
\usepackage{xifthen}
\usepackage{xargs}
\usepackage{xspace}
\usepackage{enumerate}

\usepackage{tikz}
\usetikzlibrary{backgrounds}
\usetikzlibrary{arrows}
\usetikzlibrary{shapes,shapes.geometric,shapes.misc}

\tikzstyle{tikzfig}=[baseline=-0.25em,scale=0.5]

\pgfkeys{/tikz/tikzit fill/.initial=0}
\pgfkeys{/tikz/tikzit draw/.initial=0}
\pgfkeys{/tikz/tikzit shape/.initial=0}
\pgfkeys{/tikz/tikzit category/.initial=0}

\pgfdeclarelayer{edgelayer}
\pgfdeclarelayer{nodelayer}
\pgfsetlayers{background,edgelayer,nodelayer,main}

\tikzstyle{none}=[inner sep=0mm]

\tikzstyle{rectangle}=[fill=white, draw=black, shape=rectangle]
\tikzstyle{circle}=[fill=white, draw=black, shape=circle]
\tikzstyle{vertex}=[fill=black, draw=none, shape=circle]
\tikzstyle{textbox}=[fill=white, draw=none, shape=rectangle]

\tikzstyle{line}=[-, fill=none]
\tikzstyle{rightarrow}=[->]
\tikzstyle{leftarrow}=[fill=none, <-]


%% file: macro.tex
\newcommand{\Bcal}{\mathcal{B}}

\newcommand{\Dcal}{\mathcal{D}}

\newcommand{\Fcal}{\mathcal{F}}
\newcommand{\Gcal}{\mathcal{G}}
\newcommand{\Hcal}{\mathcal{H}}

\newcommand{\Pcal}{\mathcal{P}}

\newcommand{\Rcal}{\mathcal{R}}

\newcommand{\Ucal}{\mathcal{U}}

\newcommand{\Ycal}{\mathcal{Y}}

\newcommand{\R}{\mathbb{R}}

\newcommand{\e}{\mathrm{e}}

\newcommand{\bb}{{\bm{b}}}

\newcommand{\zb}{{\bm{z}}}

\newcommand{\Ord}{\mathrm{O}}

\newtheorem{theorem}{Theorem}
\newtheorem{lemma}{Lemma}
\newtheorem{proposition}{Proposition}
\newtheorem{assumption}{Assumption}
\theoremstyle{definition}
\newtheorem{definition}{Definition}
\newtheorem{example}{Example}

\DeclareMathOperator*{\argmin}{argmin}

\DeclareMathOperator{\poly}{poly}

\DeclareMathOperator{\E}{\mathbb{E}}

\DeclarePairedDelimiter\abs{\lvert}{\rvert}

\DeclarePairedDelimiter\iprod{\langle}{\rangle}
\let\set\relax
\DeclarePairedDelimiter\set{\{}{\}}
\let\Set\relax
\DeclarePairedDelimiterX\Set[2]{\{}{\}}{\mspace{2mu}{#1}\;\delimsize|\;{#2}\mspace{2mu}}
\DeclarePairedDelimiter\brc{[}{]}
\DeclarePairedDelimiterX\Brc[2]{[}{]}{\mspace{2mu}{#1}\;\delimsize|\;{#2}\mspace{2mu}}
\DeclarePairedDelimiter\prn{(}{)}
\DeclarePairedDelimiterX\Prn[2]{(}{)}{\mspace{2mu}{#1}\;\delimsize|\;{#2}\mspace{2mu}}

\newif\iffigure
\figurefalse